\documentclass[letterpaper]{article} % DO NOT CHANGE THIS
\usepackage{aaai24}  % DO NOT CHANGE THIS
\usepackage{times}  % DO NOT CHANGE THIS
\usepackage{helvet}  % DO NOT CHANGE THIS
\usepackage{courier}  % DO NOT CHANGE THIS
\usepackage[hyphens]{url}  % DO NOT CHANGE THIS
\usepackage{graphicx} % DO NOT CHANGE THIS
\usepackage{natbib}  % DO NOT CHANGE THIS AND DO NOT ADD ANY OPTIONS TO IT
\usepackage{caption} % DO NOT CHANGE THIS AND DO NOT ADD ANY OPTIONS TO IT
\usepackage{algorithm}
\usepackage{algorithmic}
\usepackage{algorithm}
\usepackage{algorithmic}
\usepackage{amssymb}
\usepackage{amsmath}
\usepackage{mathdots} 
\usepackage{microtype} 
\usepackage{array}
\usepackage{threeparttable}
\usepackage{array}
\usepackage{color}
\usepackage{pifont}

\usepackage{algorithm}
\usepackage{algorithmic}
\usepackage{float}
\usepackage{multirow}
\usepackage{booktabs}
\usepackage{newtxtext}
\usepackage{bm}
\usepackage{float}
\usepackage{subfigure}
\usepackage{multicol}
\usepackage{pdfpages}
\usepackage{indentfirst}
\usepackage{diagbox}
\usepackage{amsmath}
\usepackage{amsthm}

\usepackage{newfloat}
\usepackage{listings}
\floatstyle{ruled}
\newfloat{listing}{tb}{lst}{}
\floatname{listing}{Listing}
\newcommand{\xmark}{\ding{55}}%
\usepackage{cite}

\usepackage{url}

\usepackage{amssymb}
\usepackage{amsmath}
\usepackage{mathdots} 
\usepackage{microtype} 
\usepackage{array}
\usepackage{threeparttable}
\usepackage{array}
\usepackage{color}
\makeatletter
\newcommand{\thickhline}{%
    \noalign {\ifnum 0=`}\fi \hrule height 2pt
    \futurelet \reserved@a \@xhline
}

\usepackage{adjustbox}
\usepackage{amssymb}
\usepackage{natbib}

\newtheorem{myDef}{Definition}

\newtheorem{myproposition}{Proposition}

\newcommand{\tabincell}[2]{\begin{tabular}{@{}#1@{}}#2\end{tabular}}
\usepackage{newfloat}
\usepackage{listings}
\DeclareCaptionStyle{ruled}{labelfont=normalfont,labelsep=colon,strut=off} % DO NOT CHANGE THIS
\floatstyle{ruled}
\newfloat{listing}{tb}{lst}{}
\floatname{listing}{Listing}
\newcommand{\cmark}{\ding{51}}%
\usepackage{newfloat}
\usepackage{listings}
\DeclareCaptionStyle{ruled}{labelfont=normalfont,labelsep=colon,strut=off} % DO NOT CHANGE THIS
\floatstyle{ruled}
\newfloat{listing}{tb}{lst}{}
\floatname{listing}{Listing}

\title{Gated Attention Coding for Training High-performance and Efficient Spiking Neural Networks}
\author{
    Xuerui Qiu\textsuperscript{\rm 1}, Rui-Jie Zhu\textsuperscript{\rm 2}, Yuhong Chou\textsuperscript{\rm 4}, Zhaorui Wang\textsuperscript{\rm 1}, Liang-jian Deng \textsuperscript{\rm 1}\thanks{Corresponding author.}, Guoqi Li \textsuperscript{\rm 3}\thanks{Corresponding author.}
}

\affiliations{
 University of Electronic Science and Technology of China \textsuperscript{1}\\
 University of California, Santa Cruz\textsuperscript{2}\\
      Institute of Automation, Chinese Academy of Sciences\textsuperscript{3}  \\
      Xi'an Jiaotong University\textsuperscript{4}  \\
   
}

\usepackage{bibentry}
\begin{document}

\maketitle

\begin{abstract}
 Spiking neural networks (SNNs) are emerging as an energy-efficient alternative to traditional artificial neural networks (ANNs) due to their unique spike-based event-driven nature. Coding is crucial in SNNs as it converts external input stimuli into spatio-temporal feature sequences.  However, most existing deep SNNs rely on direct coding that generates powerless spike representation and lacks the temporal dynamics inherent in human vision. Hence, we introduce Gated Attention Coding (GAC), a plug-and-play module that leverages the multi-dimensional gated attention unit to efficiently encode inputs into powerful representations before feeding them into the SNN architecture. GAC functions as a preprocessing layer that does not disrupt the spike-driven nature of the SNN, making it amenable to efficient neuromorphic hardware implementation with minimal modifications. Through an observer model theoretical analysis, we demonstrate GAC's attention mechanism improves temporal dynamics and coding efficiency. Experiments on CIFAR10/100 and ImageNet datasets demonstrate that GAC achieves state-of-the-art accuracy with remarkable efficiency. Notably, we improve top-1 accuracy by 3.10\% on CIFAR100 with only 6-time steps and 1.07\% on ImageNet while reducing energy usage to 66.9\% of the previous works. To our best knowledge, it is the first time to explore the attention-based dynamic coding scheme in deep SNNs, with exceptional effectiveness and efficiency on large-scale datasets. The Code is available at https://github.com/bollossom/GAC.

\end{abstract}
\vspace{-3mm}
\section{Introduction}
Artificial neural networks (ANNs) have garnered remarkable acclaim for their potent representation and astounding triumphs in a plethora of artificial intelligence domains such as computer vision \cite{krizhevsky2017imagenet}, natural language processing \cite{hirschberg2015advances} and big data applications \cite{niu2020dual}. Nonetheless, this comes at a significant cost in terms of energy consumption. In contrast, spiking neural networks (SNNs) exhibits heightened biological plausibility \cite{maass1997networks}, spike-driven nature, and low power consumption on neuromorphic hardware, e.g., TrueNorth \cite{merolla2014million}, Loihi \cite{davies2018loihi}, Tianjic \cite{pei2019towards}. 
Moreover, the versatility of SNNs extends to various tasks, including image classification \cite{hu2021advancing, xu2023constructing}, image reconstruction \cite{qiu2023vtsnn}, and language generation \cite{zhu2023spikegpt}, although the majority of their applications currently lie within the field of computer vision.

%While recent SNN research advances have led to improved performance, e.g., normalization \cite{zheng2021going}, network architecture \cite{hu2021advancing}, loss function \cite{guo2022recdis}, it remains challenging to directly train effective and efficient SNNs from scratch.
\par
\begin{figure}[!t]
\centering
\includegraphics[scale=0.32]{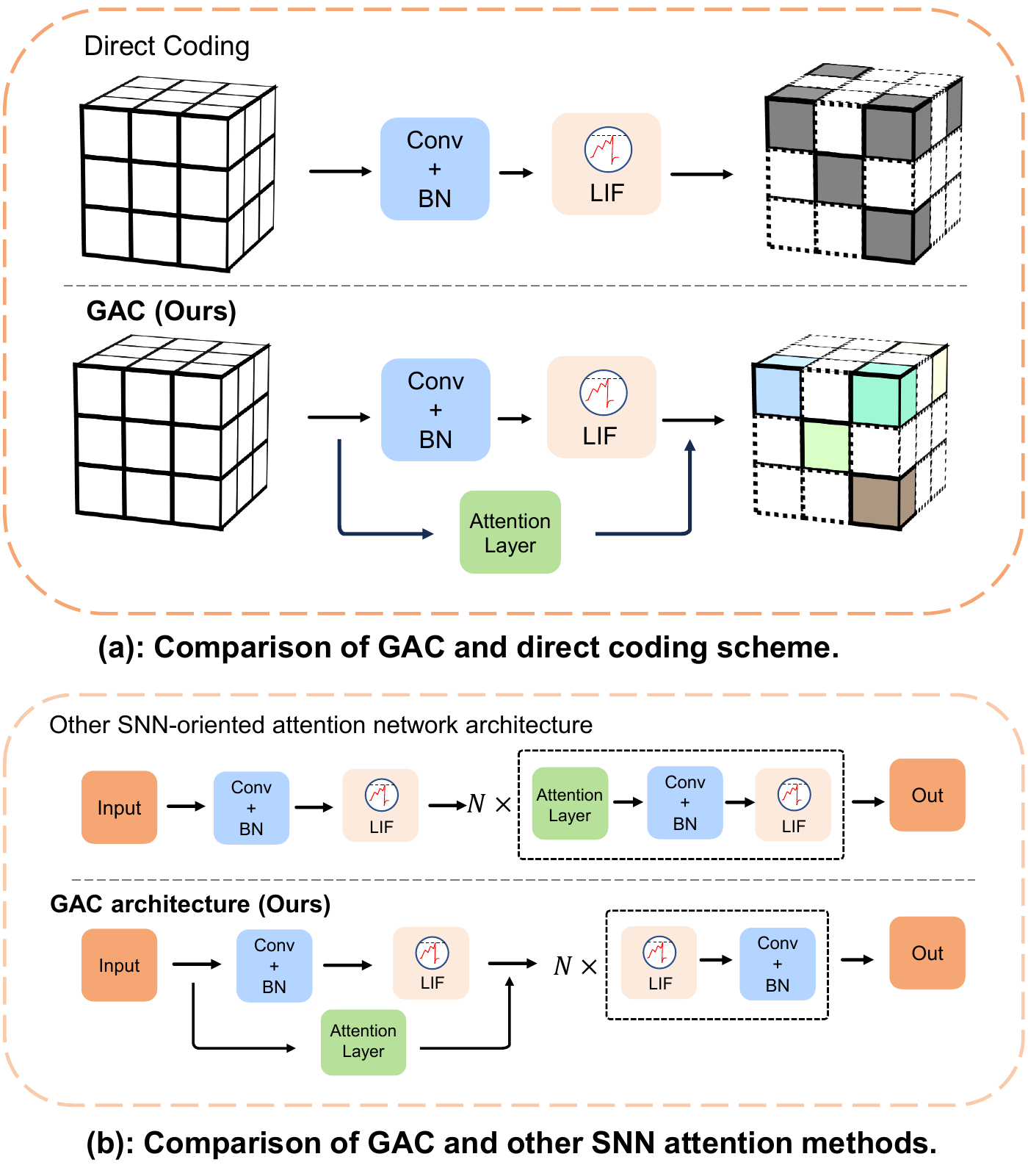}
\vspace{-3mm}
\caption{How our Gated Attention Coding (GAC) differs from existing SNNs' coding ~\cite{wu2019direct} and attention methods ~\cite{yao2021temporal, yao2023attention}. In (a), the solid-colored cube represents the float values, the gray cube denotes the binary spike values, and the cube with the dotted line represents the sparse values. In comparison with direct coding, GAC  generates spatio-temporal dynamics output with powerful representation. In (b), compared to other attention methods, GAC only adds the attention module to the encoder without requiring  $N$ Multiply-Accumulation (MAC) blocks for dynamically calculating attention scores in subsequent layers.}
\vspace{-4mm}
\label{fig:top}
\end{figure}
To integrate SNNs into the realm of computer vision, the initial challenge lies in transforming static images into spatio-temporal feature sequences. Various coding schemes have emerged to address this issue, such as rate coding \cite{van2001rate}, temporal coding \cite{comsa2020temporal}, and phase coding \cite{kim2018deep}. Among these, direct coding \cite{wu2019direct} as shown in Fig. \ref{fig:top}-(a), excel in training SNNs on large-scale datasets with minimal simulation time steps. 
% Recent SNN models incorporate diverse architectures such as spiking-ResNet~\cite{fang2021deep, hu2021advancing}, and spiking attention network~\cite{yao2021temporal,yao2023attention, zhu2022tcja}.
% , and spiking vision Transformer~\cite{zhou2023spikformer}
% These models predominantly adopt direct coding as the prevailing coding scheme in the field of SNNs.
Moreover, by adopting direct coding, recent SNN models (\citealt{li2021differentiable}; \citealt{shenESL}; \citealt{zhou2023spikformer}) achieve state-of-the-art performance across various datasets, showcasing the immense potential of coding techniques.
However, this approach utilizes a trainable layer to generate float values repetitively at each time step. The repetitive nature of direct coding leads to periodic identical outputs at every time step,  generating powerless spike representation and limiting spatio-temporal dynamics. In addition, the repetitive nature of direct coding fails to generate the temporal dynamics inherent in human vision, which serves as the fundamental inspiration for SNN models. Human vision is characterized by its ability to process and perceive dynamic visual stimuli over time. The repetitive nature of direct coding falls short in replicating this crucial aspect, emphasizing the need for alternative coding schemes that can better emulate the temporal dynamics observed in human vision. 
\par
Humans can naturally and effectively find salient regions in complex scenes \cite{itti1998model}. Motivated by this observation, attention mechanisms have been introduced into deep learning and achieved remarkable success in a wide spectrum of application domains, which is also worth to be explored in SNNs as shown in Fig. \ref{fig:top}-(b) \cite{yao2021temporal,yao2023attention}. However, it has been observed that implementing attention mechanisms to directly modify membrane potential and dynamically compute attention scores for each layer, rather than using static weights, disrupts the fundamental asynchronous spike-driven communication in these methods. Consequently, this approach falls short of providing full support for neuromorphic hardware. 
\par
 In this paper, we investigate the shortcomings of traditional direct coding and introduce an innovative approach termed Gated Attention Coding (GAC)  as depicted in Fig. \ref{fig:top}. Instead of producing periodic and powerless results, GAC leverages a multi-dimensional attention mechanism for gating to elegantly generate powerful temporal dynamic encodings from static datasets.  As a preprocessing layer, GAC doesn't disrupt the SNNs' spike-driven, enabling efficient neuromorphic hardware implementation with minimal modifications. Experimental results demonstrate that our GAC not only significantly enhances the performance of SNNs, but also notably reduces latency and energy consumption. Moreover, our main contributions can be summarized as follows:
\begin{itemize}
    \item  We propose an observer model to theoretically analyze direct coding limitations and introduce the GAC scheme, a plug-and-play preprocessing layer decoupled from the SNN architecture, preserving its spike-driven nature.
    \item We evaluate the feasibility of GAC and depict the encoding result under both direct coding and GAC setups to demonstrate the powerful representation of GAC and its advantage in generating spatio-temporal dynamics.
    \item We demonstrate the effectiveness and efficiency of the proposed method on the CIFAR10/100 and ImageNet datasets with spike-driven nature. Our method outperforms the previous state-of-the-art and shows significant improvements across all test datasets within fewer time steps. 
\end{itemize}
 % In this paper, we investigate the shortcomings of traditional direct coding and introduce an innovative approach termed Gated Attention Coding (GAC)  as depicted in Fig. \ref{fig:top}. Instead of exhibiting the periodic and powerless outputs, GAC makes full use of every moment’s input, employing a multi-dimensional attention mechanism for gating and it elegantly generates powerful spatio-temporal dynamics encoding results for static datasets while minimizing redundancy.  Experimental results demonstrate that our GAC not only significantly enhances the performance of SNNs, but also notably reduces latency and energy consumption. Moreover, our main contributions can be summarized as follows: Moreover,  As a preprocessing layer, GAC does not disrupt the spike-driven SNN, enabling easy neuromorphic hardware implementation.
 
\section{Related Works}
\subsection{Bio-inspired Spiking Neural Networks} 
\vspace{-1mm}
Spiking Neural Networks (SNNs) offer a promising approach to achieving energy-efficient intelligence. These networks aim to replicate the behavior of biological neurons by employing binary spiking signals, where a value of 0 indicates no activity and a value of 1 represents a spiking event. The spike-driven communication paradigm in SNNs is inspired by the functionality of biological neurons and holds the potential for enabling energy-efficient computational systems ~\cite{roy2019towards}. By incorporating advanced deep learning and neuroscience knowledge, SNNs can offer significant benefits for a wide range of applications \cite{jin2022sit, qiu2023when, qiu2023vtsnn}. Recently, there exist two primary methods of training high-performance SNNs. One way is to discretize ANN into spike form through neuron equivalence \cite{li2021free, hao2023reducing}, i.e., ANN-to-SNN conversion, but this requires a long simulation time step and boosts the energy consumption. We employ the direct training method \cite{wu2018spatio} and apply surrogate gradient training. 

\vspace{-2mm}
\subsection{SNN Coding Schemes}
\vspace{-1mm}
Numerous coding schemes are proposed for image classification tasks. Phase coding \cite{kim2018deep} used a weighted spike to encode each pixel and temporal coding (\citealt{park2020t2fsnn}; \citealt{comsa2020temporal}; \citealt{zhou2021temporal})  represents information with the firing time of the first neuron spike. These methods have been successfully applied to simple datasets with shallow networks, but achieving high performance becomes more challenging as datasets and networks become larger and more complex. To address this issue, rating coding \cite{van2001rate}, which encodes each pixel using spike firing frequency, has been suggested. However, it suffers from long time steps to remain high performance, while small time steps result in lower representation resolution. To overcome these limitations, \citeauthor{wu2019direct} (\citeyear{wu2019direct}) proposed the direct coding, in which input is given straight to the network without conversion to spikes and image-spike encoding is done by the first \{\textit{Conv-BN}\} layer. Then repeat this procedure at each time step and feed the results to spiking neurons. Finally, these encoded spikes will be sent to the SNNs' architecture for feature extraction.  However, the limited powerless spike representation in SNNs using direct coding leads to parameter sensitivity and subpar performance. The repetition operation fails to generate dynamic output and neglects redundant data, thus underutilizing the spatio-temporal extraction ability of subsequent SNN architectures and increasing energy consumption in neuromorphic hardware.
% \par

\begin{figure*}[!t]
\centering
\includegraphics[scale=0.122]{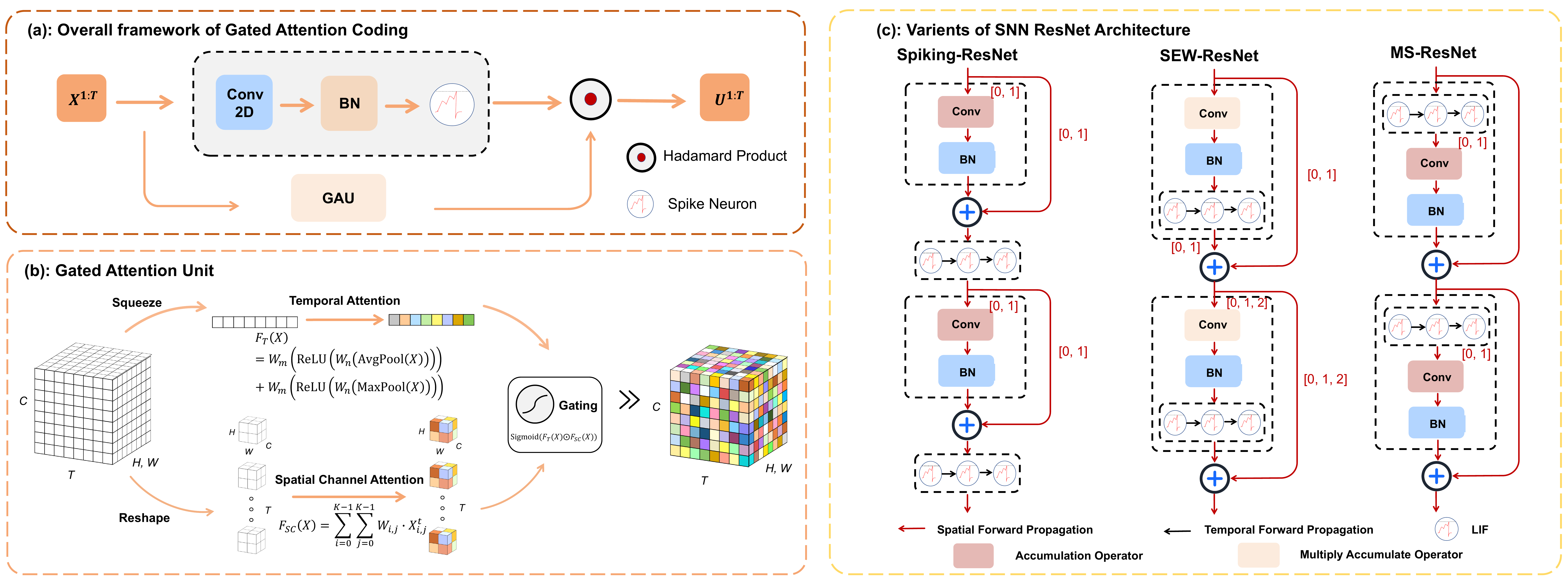}
\vspace{-3mm}
\caption{The GAC-SNN framework consists of two main components: an encoder and an architecture. In (a), we introduce the encoder, i.e., the  GAC module. (b) focuses on the GAU, which acts as the fundamental building block of the GAC layer. It comprises Temporal Attention, Spatial Channel Attention, and Gating sub-modules. (c) Common SNN ResNet architectures. The  Conv layer in SEW-ResNet uses a multiply-accumulate operator, not spike computations. Spiking-ResNet retains its spike-driven nature via direct coding, while GAC disrupts it. More details can be seen in discussions. MS-ResNet avoids floating-point multiplications, preserving its spike-driven nature. Hence, we use the MS-ResNet to benefit from neuromorphic  implementations.}
\vspace{-4mm}
\label{fig:main}
\end{figure*}

\vspace{-2mm}
\subsection{Attention Mechanism}
\vspace{-1mm}
Initially introduced to enhance the performance of sequence-to-sequence tasks \cite{bahdanau2014neural}, the attention mechanism is a powerful technique that enables improved processing of pertinent information \cite{ioffe2015batch}. By effectively filtering out distracting noise, the attention mechanism facilitates more focused and efficient data processing, leading to enhanced performance in various applications.
\citeauthor{yao2021temporal} (\citeyear{yao2021temporal}) attach the squeeze-and-excitation \cite{hu2018squeeze} attention block to the temporal-wise input of SNN, assessing the significance over different frames during training and discarding irrelevant frames during inferencing. 
However, this method only gets better performance on small datasets with shallow networks. \citeauthor{yao2023attention} (\citeyear{yao2023attention})  switch CBAM attention \cite{woo2018cbam} to multi-dimension attention and inject it in SNNs architecture,  revealing deep SNNs' potential as a general architecture to support various applications. Currently, integrating attention blocks into SNN architectures using sparse addition neuromorphic hardware poses challenges, as it necessitates creating numerous multiplication blocks in subsequent layers to dynamically compute attention scores, which could impede the spike-driven nature of SNNs. A potential solution to address this issue involves confining the application of attention mechanisms solely to the encoder, i.e., the first layer of the SNNs. By limiting the attention modifications to the initial stage, the subsequent layers can still maintain the essential spike-driven communication. This approach holds promise in enabling a more feasible implementation of  SNNs on neuromorphic hardware, as it mitigates the incompatibility arising from dynamic attention mechanisms throughout the architecture.
 % Currently, integrating attention blocks into SNN architectures is difficult because it necessitates designing separate multiplicative modules in subsequent layers to dynamically calculate attention scores. This impedes the spike-driven nature of inherently additive SNNs. Hence, the seamless integration of SNNs with neuromorphic hardware is hindered, primarily due to the hardware's limited support for static weights, which is a key requirement for conventional attention implementations.

% Presently, the incorporation of attention blocks into the architecture of  SNNs poses a challenge as it disrupts the inherent spike-driven nature. This disruption arises from the dynamic modification of the membrane potential through the application of attention mechanisms at each layer.

% Consequently, the seamless integration of SNNs with neuromorphic hardware is hindered, primarily due to the hardware's limited support for static weights, which is a key requirement for conventional attention implementations.
\section{Method}
In this section,  we first introduce the iterative spiking neuron model. Then we proposed the  Gated Attention Coding  (GAC) and Gated Attention Unit (GAU) as the basic block of it. Next, we provide the overall framework for training GAC-SNNs. Moreover, we conduct a comprehensive analysis of the direct coding scheme and explain why our GAC outperforms in generating spatio-temporal dynamics encoding results.
\vspace{-3mm}
\subsection{Iterative Spiking Neuron Model}
\vspace{-1mm}
We adopt the Leaky Integrate-and-Fire (LIF)  spiking neuron model and translate it to an iterative expression with
the Euler method (\citealt{wu2018spatio}; \citealt{yao2021temporal}). Mathematically,  the LIF-SNN layer can be described as an explicitly iterable version for better computational traceability:
% \subsection{Conv-based LIF-SNN Layer}
% Solving this differential Eq. \ref{diff}, the LIF-SNN layer can be described as an explicitly iterable version for better computational traceability (\citealt{wu2018spatio}: \citealt{yao2021temporal}):
\begin{equation}
\begin{cases}
\boldsymbol{U}^{t,n}=\boldsymbol{H}^{t-1,n}+f(\boldsymbol{W^n},\boldsymbol{X}^{t,n-1}) \\
\boldsymbol{S}^{t,n}=\Theta(\boldsymbol{U}^{t,n}-\boldsymbol{V}_{th})\\
    \boldsymbol{H}^{t,n}= \tau \boldsymbol{U}^{t,n}\cdot (1-\boldsymbol{S}^{t,n})+\boldsymbol{V}_{reset}\boldsymbol{S}^{t,n},
\end{cases}
\label{eq:lif}
\end{equation}
where $\tau$ is the time constant, $t$ and $n$ respectively represent the indices of the time step and the $n$-th layer, $\boldsymbol W$ denotes synaptic weight matrix between two adjacent layers, $f(\cdot) $ is the function operation stands for convolution or fully connected, $\boldsymbol X$ is the input, and  $\boldsymbol{\Theta(\cdot)}$ denotes the Heaviside step function. When the membrane potential $\boldsymbol{U}$ exceeds the firing threshold $\boldsymbol{V}_{th}$, the LIF neuron will trigger a spike $\boldsymbol S$. Moreover,  $\boldsymbol H$ represents the membrane potential after the trigger event which equals to $\tau  \boldsymbol{U} $ if no spike is generated and otherwise equals to the reset potential $\boldsymbol{V}_{reset}$.
\vspace{-2.5mm}
\subsection{Gated Attention Unit (GAU)}
\vspace{-1mm}
\textbf{Temporal Attention.} To establish temporal-wise
relationships between SNNs' input, we first perform the squeezing step on the spatial-channel feature map of the repeated input $\boldsymbol{X} \in \mathbb{R}^{T\times C \times H \times W}$, where $T$ is the simulation time step and $C$ is the channel size.  
Then we use Avgpool and Maxpool $\in \mathbb{R}^{T\times 1 \times 1 \times 1}$ to calculate the maximum and average of the input in the last three dimensions. Additionally, we use a shared MLP network to turn both average-pooled and max-pooled features into a temporal  weight vector $\boldsymbol{M}\in \mathbb{R}^{T}$, i.e.,  
\begin{equation}
\begin{aligned}
\label{eq1}
    \mathcal{F}_{T}( \boldsymbol{X})=& \left (\boldsymbol{W}_{m}(\mathrm{ReLU}(\boldsymbol{W}_{n}(\mathrm{AvgPool}(\boldsymbol{X})))\right)  \\&
    +\boldsymbol{W}_{m}(\mathrm{ReLU}(\boldsymbol{W}_{n}(\mathrm{MaxPool}(\boldsymbol{X})))),
\end{aligned}
\end{equation}
where $\mathcal{F}_{T}(\cdot)$ is the functional operation of temporal attention. And $\boldsymbol{W}_{m}\in\mathbb{R}^{T\times\frac{T}{r}}$, $\boldsymbol{W}_{n} \in\mathbb{R}^{\frac{T}{r}\times T}$ are the weights of two shared dense layers. Moreover, $r$  is the temporal dimension reduction factor used to manage its computing overhead. 
\par
% To make full use of the spatial channel dynamics information between SNNs' input $\boldsymbol{X} = [\boldsymbol{X}^{1}, \boldsymbol{X}^{2}, \cdots, \boldsymbol{X}^{t}]$,
\textbf{Spatial Channel Attention.} To generate the spatial channel dynamics for  encoding result, we  use a shared 2-D convolution operation at each time step to get the spatial channel matrix $\boldsymbol{N} = [\boldsymbol{N}^{1}, \boldsymbol{N}^{2}, \cdots, \boldsymbol{N}^{t}] \in \mathbb{R}^{T\times C \times H \times W}$, i.e.,
\begin{equation}
    \mathcal{F}_{SC}(\boldsymbol{X}) =\sum^{K-1}_{i=0}\sum^{K-1}_{j=0}\boldsymbol{W}_{i,j} \cdot \boldsymbol{X}_{i,j}^{t},
\end{equation}
where $\mathcal{F}_{SC}(\cdot)$ is the functional operation of spatial channel attention, $\boldsymbol{W}_{i,j}$ is the learnable parameter and $K$ represents the size of the 2-D convolution kernel size. 
\par
\textbf{Gating.}
After the above two operations, we get the temporal vector $\boldsymbol{M} $ and spatial channel matrix $\boldsymbol{N}$. Then to extract SNNs' input $\boldsymbol{X}$ temporal-spatial-channel fused dynamics features, we first broadcast the temporal vector  to  $\mathbb{R}^{T\times 1\times1 \times 1}$ and gating the above result by :
\begin{equation}
    \mathcal{F}_{G}(\boldsymbol{X})  = \sigma(\boldsymbol{M} \odot \boldsymbol{N})=\sigma(\mathcal{F}_{T}(\boldsymbol{X}) \odot \mathcal{F}_{SC}(\boldsymbol{X})),
\end{equation}
where  $\sigma(\cdot)$ and $\odot$ are the Sigmoid function and Hadamard Product.  By the above three sub-modules, we can get the functional operation $\mathcal{F}_{G}(\cdot)$  of GAU, which is the basic unit of the next novel coding.
\vspace{-2mm}
\subsection{Gated Attention Coding  (GAC)}
\vspace{-1mm}
Compared with the previous direct coding (\citealt{wu2019direct}; \citealt{hu2021advancing}), we introduce a novel encoding called Gated Attention Coding  (GAC).  And Fig. \ref{fig:main} describes the specific process.  Given that the input $\boldsymbol X \in \mathbb{R}^{C \times H \times W} $ of static datasets, we assign the first layer as an encoding layer. And we first use the convolution layer to generate features, then repeat this
procedure after each time step and feed the results to the LIF model and the GAU module respectively. Finally, gating the output of the above two modules. So to this end, the whole  GAC process can be described as:
\begin{equation}
    \label{eq9}
     \boldsymbol{O} = \mathcal{F}_{G}(f^{k\times k}(\boldsymbol{X}))  \odot \mathcal{SN}(f^{k\times k}(\boldsymbol{X})),
\end{equation}
where $\boldsymbol{X}$ and $\boldsymbol{O}$ is the GAC-SNN's input and output, $f^{k\times k}(\cdot)$ is a shared 2-D convolution operation with the filter size of $k\times k$, and $\mathcal{SN}(\cdot)$ is the spiking neuron model. Moreover, $\odot$ is the  Hadamard Product, and $ \mathcal{F}_{G}(\cdot)$ is the functional operation of GAU, which can fuse temporal-spatial-channel information for better encoding feature expression capabilities.
\vspace{-2mm}
\subsection{Overall Training Framework}
\vspace{-1mm}
We give the overall training algorithm of
GAC for training deep SNNs from scratch with our GAC and spatio-temporal backpropagation (STBP) \cite{wu2018spatio}. In the error backpropagation, we suppose the last layer as
the decoding layer, and the final output $\boldsymbol K$ can be determined by: $\boldsymbol K = \frac{1}{T} \sum_{t=1}^{T} \boldsymbol O^{t} $, where $\boldsymbol O^{t} $ is the SNNs' output of the last layer and $T$ is the time steps. Then we calculate the cross-entropy loss function \cite{rathi2020diet} between output and label, which can be described as:
\begin{alignat}{1}
   q_i&=\frac{e^{k_i}}{\sum_{j=1}^ne^{k_j}},
    \\
    \mathcal{L}&=-\sum_{i=1}^n y_i log(q_i),
\end{alignat}
where $\boldsymbol K$ = $(k_{1}, k_{2}, \cdots , k_{n})$  and $\boldsymbol Y$ = $(y_{1}, y_{2}, \cdots , y_{n})$ are the output vector and  label vector.  Moreover, the codes of the overall training algorithm can be found in \textbf{Supplementary Material A.}
\begin{figure*}[!t] \centering       
\includegraphics[width=1.6\columnwidth]{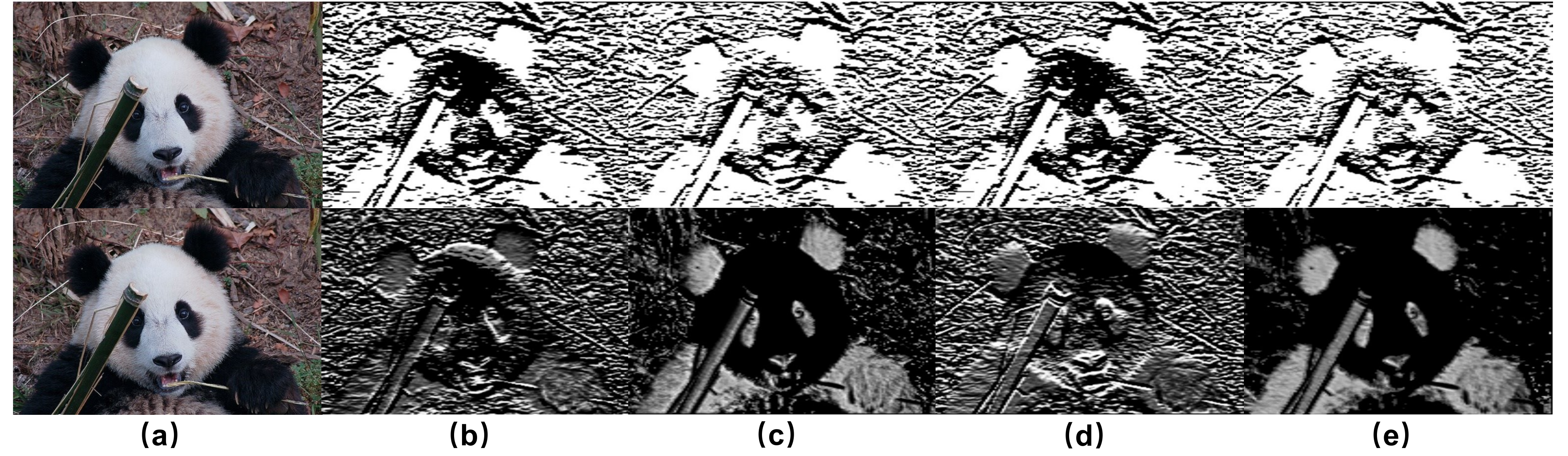} 
\vspace{-3mm}
\caption{Visualization results. (a) Original image. (b)(c)(d)(e) Encoding results of the direct coding (top) and GAC (bottom) at different time steps. Compared to direct coding, GAC enhances dynamics by introducing variations at each time step.}  
\label{vis}  
\vspace{-4mm}
\end{figure*}
\vspace{-2mm}
\subsection{Theoretical Analysis}
\vspace{-1mm}
To understand the highlights of our proposed method and the role of SNN encoders, we introduce the observer model for measuring SNN coding. Encoders are used to convert static images into feature sequences, incorporating temporal information into SNNs at each time step. Some encoders are embedded within the SNN as part of it (e.g., the first Conv-based spiking neuron layer for direct coding), while others are not included in the SNN models, e.g., rate coding. The embedded encoders can be easily distinguished from the rest of the network since they use actual values for linear transformations, unlike spikes or quantized data. Functionally, encoders convert static data into the temporal dimension. This definition helps us understand what SNN encoders are.
\begin{myDef}
    \label{def1}
    \textbf{Encoder.} An encoder in SNNs for image classification tasks is used to translate static input $\boldsymbol{X} \in \mathbb{R}^{ C_{in}\times H\times W}$ into dynamics feature sequences $\boldsymbol{A} \in \mathbb{R}^{ T\times C_{out}\times H\times W}$.
\end{myDef}
 Moreover, two points should be noted in \textbf{Definition} 1. Firstly, $\boldsymbol{A}$ is used to indicate the encoders' output no matter spikes or real values. And Membrane Shortcut (MS) ResNet architecture \cite{hu2021advancing} is considered to use sequential real values as input after the encoder. Secondly, although two dimensions ($C_{out}, T$) are changed, the spatial one $C_{out}$ is similar to the operation in ANNs, which means $T$ is unique for SNN encoders. In other words, the time step $T$ is the secret of the time information, and the SNN encoder is designed to generate this added dimension. The discussion above implies that to understand and metric an encoder, we should focus on its temporal dimension and find a proper granularity.
\begin{myDef}
    \label{def2}
    \textbf{Neuron Granularity.} Considering $\boldsymbol{A}\in\mathbb{R}^{T\times C_{out}\times H\times W}$ is output feature sequences of the encoder, given a fixed position $c, h, w$ for the output $\boldsymbol{A}$, so that we get a vector along temporal axis $a = \boldsymbol{A}_{:, c, h, w}=\left[a^{1}, a^{2}, \cdots, a^{t}\right]$.
\end{myDef}
Here the encoding feature vector $a$  is not subscripted in \textbf{Definition} 2 because the encoder is usually symmetric, and the choice of position for analysis does not affect its generality. To measure the vector $a$, we introduce the observer model and information entropy \cite{ash2012information}. Assuming an observer is positioned right behind the encoder, recording and predicting the elements of vector $a$  in a time-ordered manner. Hence, the observer model can be formally established as follows:
% Here the vector $a$ is not subscripted since the encoder is usually symmetry and no matter what each position is chosen for analysis, the generality is not lost. To metric the vector $a$, the observer model and information entropy \cite{ash2012information} are then introduced. Since the sequential data pops out in a time order, supposing an observer is right behind the encoder, recording and predicting the elements of vector $a$ in a time order. Hence, the observer model can be formally established as follow:
\begin{itemize}
    \item  The observer notices the elements of encoded feature sequences $a \in \mathbb{R}^{T}$ or $\{0, 1\}^{T}$ in time step order.
    \item At any time step $t$, the observer remembers all elements from time $1$ to time $t-1$, and it is guessing the element $a^{t}$ of encoded feature sequences $a \in \mathbb{R}^{T}$.
    \item The observer is aware of the mechanism or the encoder structure, but not its specific parameters.
\end{itemize}
\par
Moreover, at time step $t$, guessing $a^{t}$, the observer should answer it with probability. The probability can be described as:
\begin{equation}
    p^{t}(a^{t})=p( a^{t}|a^{1}, \cdots, a^{t-1}),
\end{equation}
\vspace{-1mm}
And we  use the information entropy to meter the quantity of information gotten by the observer, i.e.,
\begin{equation}
    \mathcal H(\boldsymbol V^{t})=\sum_{t} p^{t}(a^{t})\log ( p^{t}(a^{t})),
    \vspace{-2mm}
\end{equation}
where $\boldsymbol V^{t}$ is used to indicate the random variable version of $a^{t}$. Specifically, when $p^{t}(a^{t})=0$ or  $1$, $p^{t}(a^{t})\log ( p^{t}(a^{t}))=0$ it means that a deterministic event contribute $0$ to information entropy. Moreover, for the observer model, when the element $a^{t}$ is deterministic, there is no additional information that deserves observing at time $t$.
\par
To better understand the concept of information entropy in this context, it is crucial to consider the role of an encoder whose task is to convert information into tensors that generate temporal dynamics. Ideally, the encoder should utilize as numerous time steps as possible to code information, resulting in a positive information entropy along the time axis. The positive entropy indicates the presence of information, which is crucial for spiking neural networks. While it is difficult to assign a precise value to the entropy, it is possible to measure the duration of positive entropy. In this way, longer-lasting positive entropy can be considered a more effective use of the temporal dimension.
\par
\begin{myDef}
\label{def3}
\textbf{Dynamics Loss \& Dynamics Duration.} Considering a specific position with encoded feature vector $a = [a^{1}, a^{2}, \cdots, a^{t}]$ alone temporal dimension, if there exists $t_{e}$ for all $t$, $t>t_{e}$, the observer mentioned above have an entropy $\mathcal{H}(\boldsymbol V^{t})=0$. Then we call the moment $t$ after $t_{e}$ is \textbf{Dynamics Loss}. And for $\boldsymbol T_{e}=\inf{(t_{e})}$, we call it \textbf{Dynamics Duration}.
\end{myDef}
\textbf{Definition} 3 delineates the encoder's effective encoding range. Dynamics Duration indicates when coding entropy $\mathcal{H}(\boldsymbol V^{t})\geq$ 0. At Dynamics Loss time steps, the entropy  $\mathcal{H}(\boldsymbol V^{t})$  drops to 0, rendering encoding unnecessary.  Moreover, to metric the encoder, the key is to find and compare the dynamics duration time step $\boldsymbol T_{e}$.
% Proposition 1 The \textbf{Definition}\ref{def3} distinguishes the effective range of the sequence encoded by the encoder. Dynamics Duration denotes the interval where the entropy of possible coding results exceeds 0, containing more meaningful encodings. For an observer model, at the time step of Dynamics Loss, the information entropy is 0, which means that encoding is dispensable at these moments. Moreover, to metric the encoder, the key is to find and compare the dynamics duration.
\begin{myproposition}
\label{lemma1}
% T_{e}
    Given same \{\textit{Conv-BN}\} parameters, denoting the dynamic duration of GAC as $\boldsymbol T_{g}$ and direct coding's as $\boldsymbol T_{d}$, and $\boldsymbol T_{g} \geq \boldsymbol T_{d}$
\end{myproposition}
\vspace{-3mm}
\begin{proof} 
Denoted that $\boldsymbol X\in \mathbb{R}^{T\times C\times H\times W} $ is the repetitive output after \{\textit{Conv-BN}\} module, $a = [a^{1}, a^{2}, \cdots, a^{t}]$ is  the encoded feature vector. 
\par
For direct coding, it sends the repetitive output $\boldsymbol X$ to the spiking neuron for coding, resulting in the encoded feature vector $a$ being powerless and periodic with 0 or 1.
% For direct coding, it sends the repetitive results $X$ to the spiking neuron for coding, and its output encoding feature vector $a$ is periodic, with 0 or 1.
% For direct coding where repetitive input is given for the spiking neuron for the last step of coding, the output encoded feature vector $a$ is periodic with 0 or 1. 
Moreover, the period $\boldsymbol T_{p}$ is: 
\vspace{-1mm}
\begin{equation}
    \boldsymbol T_{p}=\lceil\log_{\tau}(1 - \frac{\boldsymbol V_{th}(1- \tau )}{x_{i,j}})\rceil,
\end{equation}
where $\tau$ is the  time constant, $\boldsymbol V_{th}$ is the firing threshold and $x_{i,j}$ is the pixel of the input $\boldsymbol X$ after \{\textit{Conv-BN}\} module.  Hence, the subsequent output is predictable when the observer has found the first spike. Thus, the direct coding's dynamic duration  $\boldsymbol T_{d}=\boldsymbol T_{p}$. Moreover, the derivation of $\boldsymbol{T}_p$ and analysis of other coding schemes' dynamics can be found in \textbf{Supplementary Material B}.
\par
For GAC, we multiplied the output of direct coding and GAU to expand the dynamic duration of the encoding results. Thus, the  GAC's dynamic duration $\boldsymbol T_{g}=\lfloor\frac{\boldsymbol T}{\boldsymbol T_{d}}\rfloor \boldsymbol T_{d}$. It can be seen that $\boldsymbol T_{g} \geq \boldsymbol T_{d}$.
% Furthermore, if $\boldsymbol T\gg \boldsymbol T_{d}$, then $\boldsymbol T_{g}=\lfloor\frac{\boldsymbol T}{\boldsymbol{T}_{d}}\rfloor \boldsymbol{T}_{d}\approx \boldsymbol{T}$ and $T_{g}\gg T_{d}$
% For our GAC, the periodic output from the direct coding-like structure is Hadamard produced by real value. The operation devastates the periodicity of output.
\end{proof}
\vspace{-2mm}
According to \textbf{Proposition} \ref{lemma1}, GAC lasts its dynamics longer than direct coding. Moreover, this reflects the superiority of GAC in generating dynamic encoding results. As depicted in Fig. \ref{vis}, GAC's encoding results on static datasets vary significantly at each time step, i.e., temporal dynamics.
\vspace{-2mm}
\subsection{Theoretical  Energy Consumption Calculation}
\vspace{-1mm}
% GAC controls membrane potentials and reduces spike firing rates of the post-encoding layer in GAC-SNN. It achieves this by transforming matrix multiplication into sparse addition, which can be implemented as addressable addition on neuromorphic chips.
 The GAC-SNN architecture can transform matrix multiplication into sparse addition, which can be implemented as addressable addition on neuromorphic chips.
In the encoding layer, convolution operations serve as MAC operations that convert analog inputs into spikes, similar to direct coding-based SNNs \cite{wu2019direct}. Conversely, in SNN's architecture, the convolution (Conv) or fully connected (FC) layer transmits spikes and performs AC operations to accumulate weights for postsynaptic neurons. Additionally, the inference energy cost of GAC-SNN can be expressed as follows:
\vspace{-2mm}
\begin{equation}
\label{energy}
    \begin{aligned}
E_{total}&=E_{MAC}\cdot FL_{conv}^1+\\&E_{AC}\cdot T \cdot(\sum_{n=2}^N FL_{conv}^n \cdot fr^{n} +\sum_{m=1}^M FL_{fc}^m \cdot fr^{m}),
\end{aligned}
\end{equation}
where $N$ and $M$ are the total number of Conv and FC layers, $E_{MAC}$ and $E_{AC}$ are the energy costs of MAC and AC operations, and $fr^{m}$, $fr^{n}$, $FL_{conv}^n$ and $FL_{fc}^m$ are the firing rate and FLOPs of the $n$-th Conv and $m$-th FC layer.  Previous SNN works  (\citealt{horowitz20141}; \citealt{rathi2020diet})  assume 32-bit floating-point implementation in 45nm technology, where $E_{MAC}$ = 4.6pJ and $E_{AC}$ = 0.9pJ for various operations.

\begin{table*}[!t]
   \centering
   \caption{Comparison between the proposed methods and previous works on CIFAR datasets. $^{\ddagger}$ denote self-implementation results with  open-source code \cite{hu2021advancing}.The "Params" column indicates network parameter size on  CIFAR10/100 datasets.}
   \vspace{-2.5mm}
   \begin{adjustbox}{max width=\linewidth} 
   \begin{tabular}{cccccc}
   \toprule 
   \multicolumn{1}{c}{ Methods} &\multicolumn{1}{c}{  Architecture}
&\multicolumn{1}{c}{  \tabincell{c}{Params\\ (M)}}
&\multicolumn{1}{c}{ \tabincell{c}{Time\\Steps}}
&\multicolumn{1}{c}{   \tabincell{c}{CIFAR10\\  Acc.(\%)}}
&\multicolumn{1}{c}{   \tabincell{c}{CIFAR100\\  Acc.(\%)}}\\
    \midrule
ANN2SNN \cite{hao2023reducing}\textit{\textsuperscript{AAAI}}                  & VGG-16   &    33.60/33.64               & 32               & 95.42  &  76.45               \\
tdBN  \cite{zheng2021going}\textit{\textsuperscript{AAAI}}                      & Spiking-ResNet-19     &    12.63/12.67               & 6               & 93.16 & 71.12                \\
% Dspikes \cite{li2021differentiable}\textit{\textsuperscript{NeurIPS}}                      & Spiking-ResNet-18     & 12.57               & 6               & 93.16 & 74.24                \\
    
TET \cite{deng2022temporal}\textit{\textsuperscript{ICLR}}                       & Spiking-ResNet-19    & 12.63/12.67                 & 6               & 94.50 & 74.72                \\
    
% RecDis \cite{guo2022recdis}\textit{\textsuperscript{CVPR}}                      & Spiking-ResNet-19     & 12.63                & 6               & 94.71 & 74.10                \\
   
GLIF \cite{yaoglif}\textit{\textsuperscript{NeurIPS}}                       &Spiking-ResNet-19  &  12.63/12.67                        & 6              & 95.03 &  77.35               \\
% Spikformer \cite{zhou2023spikformer}\textsuperscript{\textit{ICLR}}
%  & \ding{55}             & Spiking-ViT                     & 4               & 95.51 & 4&   78.21              \\
\midrule
$ \text{MS-ResNet}^{\ddagger}$ \cite{hu2021advancing}  
              & MS-ResNet-18     & 12.50/12.54                 & 6               &94.92&76.41             \\
  
    \cmidrule{2-6}
\multirow{3}{*}{\textbf{GAC-SNN}}       & MS-ResNet-18    & 12.63/12.67  
& 6 &\textbf{96.46}$\pm$0.06&\textbf{80.45}$\pm$0.27    \\

        &   MS-ResNet-18  &     12.63/12.67      & 4              &96.24$\pm$0.08  &79.83$\pm$0.15  \\
      & MS-ResNet-18  &    12.63/12.67   & 2                    &96.18$\pm$0.03 &78.92$\pm$0.10  \\
     \midrule
     $ \text{ANN}^{\ddagger}$ \cite{hu2021advancing} 
             & MS-ResNet-18       &   12.50/12.54               & N/A              &96.75 &80.67                \\
    \bottomrule
      \end{tabular}
\end{adjustbox}
\label{tab:CIFAR}
\vspace{-4mm}
\end{table*}

\vspace{-1mm}
\section{Experiments}
In this section, we evaluate the classification performance of GAC-SNN on static datasets, e.g., CIFAR10, CIFAR100, ImageNet (\citealt{li2017cifar10}; \citealt{krizhevsky2017imagenet}). To verify the effectiveness and efficiency of the proposed coding, we integrate the GAC module into the Membrane Shortcut (MS)  ResNet  \cite{hu2021advancing}, to see if the integrated architecture can generate significant improvement when compared with previous state-of-the-art works. Specifically, the details of the architecture are shown in Fig. \ref{fig:main}-(c), and why we use it is illustrated in the discussions. More details of the training details, datasets, hyper-parameter settings, convergence analysis, and trainable parameter analysis can be found in  \textbf{Supplementary Material B}.
\vspace{-2mm}
\subsection{GAC Can Produce Powerful and Dynamics Results}
\vspace{-1mm}
We evaluated GAC's effectiveness in reducing redundant temporal information and improving encoding results for static datasets. By training MS-ResNet-34 on ImageNet with and without GAC, we generated the encoding output shown in  Fig. \ref{vis}. And it can be seen that our GAC can help SNNs to capture more texture information. Hence, our approach enhances SNNs' representation ability and temporal dynamics by introducing significant variations in GAC results at each time step, compared to the periodic output of direct coding.
 % And it can be seen in GradCAM heat map results that our GAC can help SNNs to capture more texture information.
\vspace{-2mm}
\subsection{GAC Can Get Effective and Efficient SNNs}
\vspace{-1mm}
\textbf{Effectiveness.} The GAC-SNN demonstrate remarkable performance enhancement compared to existing state-of-the-art works (Tab. \ref{tab:CIFAR}-\ref{tab:imagenet}). On CIFAR10 and CIFAR100, GAC achieves higher accuracy than previous methods using only 2-time steps. With the same time steps, GAC improves 1.43\% and 3.10\% on CIFAR10 and CIFAR100 over GLIF \cite{yaoglif}. Moreover, compared to the baseline MS-ResNet \cite{hu2021advancing}, our method outperforms it on CIFAR10 and CIFAR100 by 1.54\% and 4.04\% with 6-time steps. For the larger and more challenging ImageNet dataset, compared
with the baseline MS-ResNet \cite{hu2021advancing}, we apply our GAC to MS-ResNet-18 and can significantly increase the accuracy (65.14\% v.s. 63.10\%). Compared with other advanced works \cite{fang2021deep,yao2023attention}, GAC-based MS-ResNet-34 achieves  70.42\% top-1 accuracy and surpasses all previous directly-trained SNNs with the same depth.
% In comparison with existing state-of-the-art works, the GAC-SNNs exhibit a remarkable enhancement in performance (Tab. \ref{tab:CIFAR}-\ref{tab:imagenet}). For CIFAR10 and CIFAR100 datasets, we report the mean and standard deviation of 5 runs under different random seeds as shown in  Tab. \ref{tab:CIFAR} and GAC achieves higher accuracy than previous work with only 2-time steps. Specifically,  using the same time steps, our GAC also has a significant advantage, improving 1.43\% and 3.10\% on CIFAR10 and CIFAR100 compared with  GLIF \cite{yaoglif}. Moreover, compared with the baseline MS-ResNet \cite{hu2021advancing}, our method can outperform theirs with 6 time-steps by 1.54\% and 4.04\%  on CIFAR10 and CIFAR100.
% Att-MS-ResNet \cite{yao2023attention} achieves 77.1\% accuracy with an added attention auxiliary module in MS-ResNet architecture. 

% However, this approach faces certain drawbacks. The dynamic calculation of attention scores and the consequent modification of membrane potential at each layer comes at the cost of sacrificing the spike-driven nature of the network. Additionally, this implementation requires a larger number of parameters (78.37M vs. 77.41M) and an extended training time (1000 epochs vs. 250 epochs) compared to alternative methods.

% Although Att-MS-ResNet-104 \cite{yao2023attention} achieves 77.1\% accuracy by plugging an additional attention auxiliary module in MS-ResNet architecture, it destroys the spike-driven nature and requires more parameters (78.37M vs. 77.41M) and training time (1000epoch vs. 250epoch).
\par
 \textbf{Efficiency.} Compared with prior works, the GAC-SNN shine in energy consumption (Tab. \ref{tab:imagenet}). We first make an intuitive comparison of energy consumption in the SNN field. Specifically, GAC-SNN (This work) vs. SEW-ResNet-34 at 4-time steps: Power, 2.20mJ vs. 4.04mJ. That is, our model has +3.22\% higher accuracy than SEW-ResNet with only the previous 54.5\% power. And GAC-SNN (This work) vs. MS-ResNet-34 vs. Att-MS-ResNet-34 at 6-time steps: Power, 2.34mJ vs. 4.29mJ vs. 4.11mJ. That is, Our model has the lowest power under the same structure and time steps. For instance, as the layers increase from 18 to 34, MS ResNet (baseline) has 1.83×(4.29mJ/2.34mJ) and 1.51×(5.11mJ/3.38mJ) higher energy consumption than our GAC-SNN. At the same time, our task performance on the above same depth network structures has improved by +2.04\% and +0.83\%, respectively.
 % What is more attractive is that the power of deep SNNs using GAC is further reduced. Specifically, GAC-SNNs (This work) vs. SEW-Res-SNN-152 vs. Att-MS-Res-SNN-104 vs. MS-Res-SNN-104 at 4-time step: Power, 4.67mJ   vs. 7.30mJ vs. 12.89mJ  vs. 10.19mJ. In conclusion, using GAC can greatly reduce the power of SNNs.
\begin{table*}[t]
\caption{Comparison between the proposed method and previous works on the ImageNet dataset. Power is the average theoretical energy consumption when predicting a batch of images from the test set, details of which are shown in Eq.\ref{energy}.  The "Spike-driven" column indicates if an independent design of the multiplication module is required in the SNNs architecture, which hinders the implementation of neuromorphic hardware.}
\begin{center}
\tabcolsep=0.16cm
\vspace{-4.5mm}
  \begin{adjustbox}{max width=\linewidth} 
\begin{tabular}{ccccccc}
\toprule
\multicolumn{1}{c}{ Methods} &\multicolumn{1}{c}{  Architecture}
&\multicolumn{1}{c}{  \tabincell{c}{Spike\\ -driven}}
&\multicolumn{1}{c}{  \tabincell{c}{Params\\(M)}}
&\multicolumn{1}{c}{ \tabincell{c}{Time\\Steps}}
&\multicolumn{1}{c}{  \tabincell{c}{Power\\ (mJ)}}
&\multicolumn{1}{c}{   \tabincell{c}{Top-1\\ Acc.(\%)}}\\
 \midrule
  ANN2SNN \cite{hao2023reducing}\textit{\textsuperscript{AAAI}}  & ResNet-34 & \cmark &21.79 &64&-  &68.61\\ 
    TET \cite{deng2022temporal}\textit{\textsuperscript{ICLR}} & Spiking-ResNet-34 &\cmark &21.79&6 &-  & {64.79} \\
    tdBN \cite{zheng2021going}\textit{\textsuperscript{AAAI}} &Spiking-ResNet-34 &\cmark &21.79 & 6&{6.39}  & 63.72 \\ 
   
    SEW-ResNet \cite{fang2021deep}\textit{\textsuperscript{NeurIPS}} 
    & SEW-ResNet-34 &\xmark &21.79  &4 &{4.04}  & 67.04 \\
       \midrule
    \multirow{2}{*}{MS-ResNet \cite{hu2021advancing}}  
    & MS-ResNet-18 & \cmark &11.69 &6& {4.29}  & 63.10 \\
    & MS-ResNet-34 & \cmark &21.80 &6 & {5.11}  & 69.43 \\
     \cmidrule{2-7}

    \multirow{2}{*}{Att-MS-ResNet \cite{yao2023attention}\textit{\textsuperscript{TPAMI}}}  
     & Att-MS-ResNet-18 & \xmark &11.96(+0.27)& 6 & {4.11}  & 64.15$^\star$\\
    & Att-MS-ResNet-34 & \xmark &22.30(+0.50)& 6 & {5.05} & 69.35$^\star$\\
   \cmidrule{2-7}
    \multirow{2}{*}{\textbf{GAC-SNN}} 
    & MS-ResNet-18& \cmark  &11.82(\textbf{+0.13})  &6/4 & 2.34/\textbf{1.49} & \textbf{65.14}/64.05\\
    & MS-ResNet-34 & \cmark  &21.93(\textbf{+0.13}) &6/4 & {3.38/\textbf{2.20}} & \textbf{70.42}/69.77\\
\midrule
\multirow{2}{*}{ANN \cite{hu2021advancing}}  
& MS-ResNet-18 & \xmark & {11.69}  & {N/A}& {14.26}& {69.76} \\ 
    & MS-ResNet-34 & \xmark & {21.80}& {N/A}  & {16.87} & {73.30}\\

\bottomrule
 \multicolumn{7}{l}{$^\star$ needs a large training time (1000 epochs and 600 batch size) compared to other methods.}
\end{tabular}
\label{tab:imagenet}
\end{adjustbox}
\end{center}
\vspace{-4mm}
\end{table*}

\begin{table}[htbp]
\centering
\vspace{-2.5mm}
   \caption{Comparisons with different coding schemes.}
   \vspace{-3mm}
   \begin{adjustbox}{max width=\linewidth} 
\begin{tabular}{cccc}
\toprule 

Architecture               & Schemes   & Time Steps                              & Acc.(\%) \\
\midrule
ResNet-19   & Phase Coding        & 8 & 91.40\\
VGG-16 & Temporal Coding \        &100 & 92.68\\
ResNet-19   & Rate Coding         &6&  93.16\\
MS-ResNet-18  & Direct Coding        & 6 &94.92\\
\midrule
MS-ResNet-18& GAC      & 6 & \textbf{96.46}$\pm$0.06 \\

\bottomrule 
\end{tabular}
\label{Ablation:coding}
\end{adjustbox}
\vspace{-4mm}
\end{table}

\begin{figure}[!t] \centering       
\includegraphics[width=0.95\columnwidth]{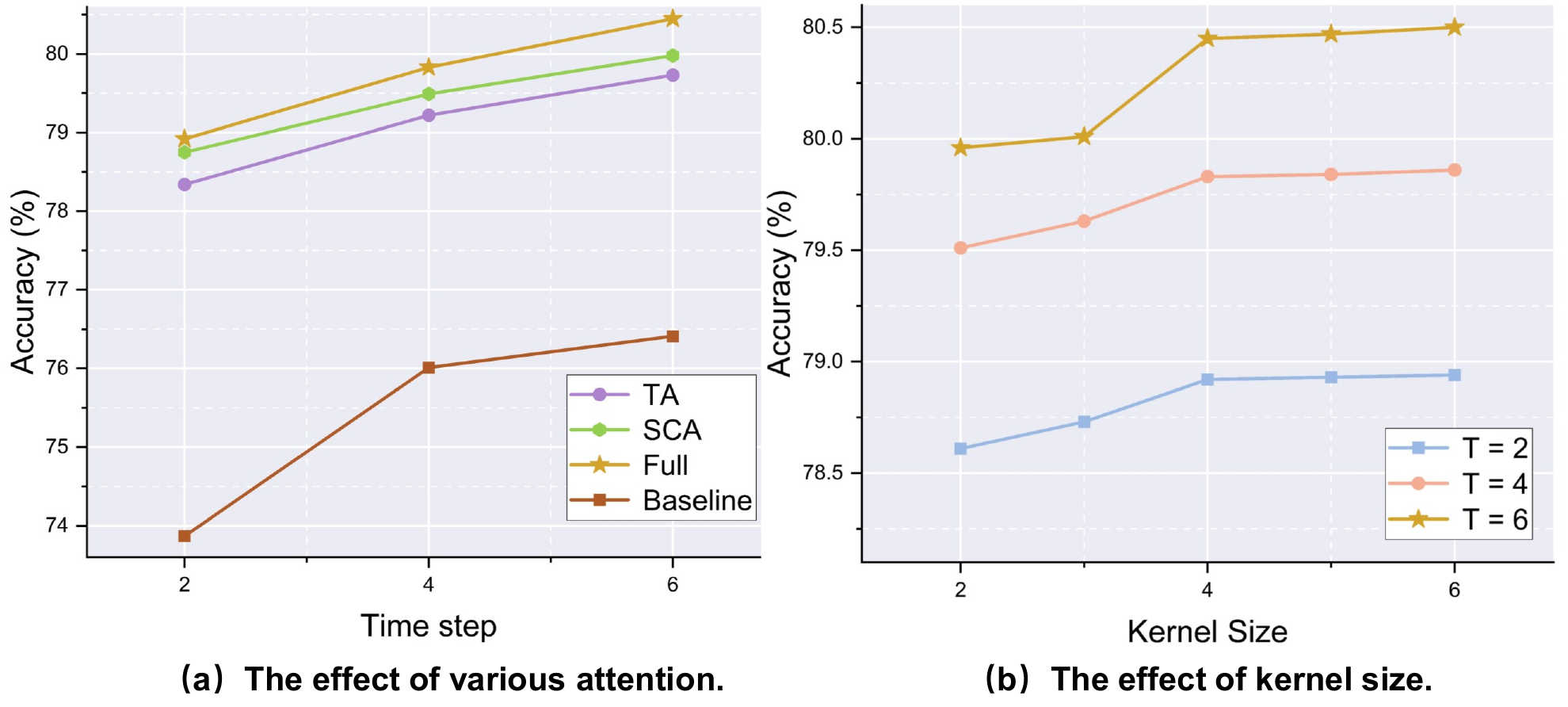} 
\vspace{-3mm}
\caption{Ablation study on CIFAR100.}  
\label{Ablation}  
\vspace{-4mm}
\end{figure}
% \begin{figure}[t] \centering    
% \subfigure[The Effect of various attention. ] { 
% \label{attention}     
% \includegraphics[width=0.47\columnwidth]{AnonymousSubmission/LaTeX/Attetnion.pdf} 
% }   
% \subfigure[The Effect of kernel size.] { 
% \label{kernel size} 
% \includegraphics[width=0.47\columnwidth]{AnonymousSubmission/LaTeX/kernel.pdf}
% }  
% \vspace{-3mm}
% \caption{Ablation study on CIFAR100.}  
% \vspace{-4mm}
% \label{fig06}     
% \end{figure}

\vspace{-2mm}
\subsection{Ablation study}
\vspace{-1mm}
\textbf{Comparison between Different SNN Coding Schemes.}
To future demonstrate the advantage of GAC, we evaluate the performance of our GAC and other coding schemes e.g., Phase coding \cite{kim2018deep}, Temporal coding \cite{zhou2021temporal}, Rate coding \cite{wu2019direct}, Direct coding  \cite{wu2019direct}. Tab. \ref{Ablation:coding} displays CIFAR10 test accuracy, where GAC achieves 96.18\% top-1 with MS-ResNet-18 in 6-time steps. 
% It also outperforms other schemes, even with fewer steps, when the network structure is the same.
\par
\textbf{The Effect of Parameter  Kernel Size $\boldsymbol K$.} 
We investigate the impact of the 2D convolution kernel size $K$ in the Spatial Channel Attention module of our GAC. Specifically, there is a trade-off between performance and latency as kernel size increases. It is almost
probable that when kernel size increases, the receptive region
of the local attention mechanism also does so, improving SNN performance. These benefits do, however, come at a cost of high parameters and significant latency. To this end, we trained the GAC-based MS-ResNet-18 on CIFAR100 with various $K$ values. As shown in Fig. \ref{Ablation}-(b), accuracy rises with increasing $K$, plateauing after $K$ exceeds 4. This indicates our GAC maintains strong generalization despite large $K$ variations. To maintain excellent performance and efficiency, we consider employing $K=4$ in our work.

\par \textbf{Comparison of Different Attention.} 
%  e. As shown in Fig. \ref{Ablation}-(a), the SCA module contributes significantly to performance enhancement. This is because time steps are much fewer than channels in most SNNs designs, allowing the SCA module to extract extra relevant features than the TA module. Notably, regardless of the module we ablate, performance will be affected, which may help you understand our design.
We conducted ablation studies on the Temporal Attention (TA) and Spatial Channel Attention (SCA) modules to assess their effects. Fig. \ref{Ablation}-(a) indicates that the SCA module contributes more to performance improvement due to the most SNNs design that channels outnumber time steps. The SCA module extracts additional significant features compared to the TA module. Notably, regardless of the module we ablate, performance will be affected, which may help you understand our design.
\vspace{-2mm}
\section{Discussions}
\vspace{-1mm}
\textbf{Analysis of Different GAC-SNN's ResNet Architecture.} Residual connection is a crucial basic operation in deep SNNs'  ResNet architecture. And there are three shortcut techniques in existing advanced deep SNNs. Spiking-ResNet \cite{hu2021spiking} performs a shortcut between membrane potential and spike. Spike-Element-Wise (SEW) ResNet \cite{fang2021deep} employs a shortcut to connect the output spikes in different layers. Membrane Shortcut (MS) ResNet \cite{hu2021advancing}, creating a shortcut between membrane potential of spiking neurons in various layers. Specifically, we leverage the membrane shortcut in the proposed GAC for this reason:
\par \textit{Spike-driven} describes the capacity of the SNNs' architecture to convert matrix multiplication (i.e., high-power Multiply-Accumulation) between weight and spike tensors into sparse addition (i.e., low-power Accumulation). The spike-driven operations can only be supported by binary spikes. However, as the SEW shortcut creates the addition between binary spikes, the values in the spike tensors are multi-bit (integer) spikes.  Additionally, GAC-based Spiking-ResNet is not entirely spike-driven. Because the second layer convolution operation's input changes to a floating-point number when using GAC, the input for the other layers remains a spike tensor.  By contrast, as shown in Fig. \ref{fig:main}-(c), spiking neurons are followed by the MS shortcut. Hence, each convolution layer in GAC-based MS-ResNet architecture will always get a sparse binary spike tensor as its input.
% \par \textit{Dynamical isometry}  provides a theoretical justification for effectively working deep neural networks \cite{chen2020comprehensive} and it has been demonstrated to be satisfied by MS-ResNet.

\par
\textbf{Impact of GAC and Other SNNs' Attention Methods on the Spike-driven Nature.} As shown in Fig. \ref{fig:top}-(b), other SNN-oriented attention works \cite{yao2021temporal, yao2023attention} adding an attention mechanism to the SNNs architecture need to design numerous multiplication blocks and prevent all matrix multiplications related to the spike matrix from being converted into sparse additions, which hinders the implementation of neuromorphic hardware. However, adding an attention mechanism in the encoder doesn't. As the encoder and architecture are decoupled in SNN hardware design \cite{li2023firefly}, our GAC, like direct coding \cite{wu2019direct}, incorporates the multiplication block for analog-to-spike conversion in the encoder without impacting the spike-driven traits of the sparse addition SNN architecture. 
% Because direct coding \cite{wu2019direct} employs a multiplicative block in the first layer to transform analog input to spikes, our GAC achieves a similar effect with this multiplication block.
% \textbf{Impact of GAC and Other SNNs' Attention Methods on the Spike-driven Nature.} As shown in Fig. \ref{fig:top}-(b),  except for the first layer, other SNN-oriented attention works \cite{yao2021temporal, yao2023attention} adding an attention mechanism to the SNNs architecture need to design numerous separate multiplication blocks and prevents all matrix multiplications related to the spike matrix from being converted into sparse additions, which hinders the implementation of neuromorphic hardware. Adding an attention mechanism in the encoding layer does not prevent implementation on neuromorphic hardware. Because direct coding \cite{wu2019direct} employs a multiplicative block in the first layer to transform analog input to spikes, our GAC achieves a similar effect with this multiplication block.
\vspace{-2mm}
\section{Conclusion}
\vspace{-1mm}
This paper focuses on the SNNs' coding problem, which is described as the inability of direct coding to produce powerful and temporal dynamic outputs. We have observed that this issue manifests as periodic powerless spike representation due to repetitive operations in direct coding. To tackle this issue, we propose Gated Attention Coding (GAC), a spike-driven and neuromorphic hardware-friendly solution that seamlessly integrates with existing Conv-based SNNs. GAC incorporates a multi-dimensional attention mechanism inspired by attention mechanism and human dynamics vision in neuroscience. By effectively establishing sptiao-temporal relationships at each moment, GAC acts as a preprocessing layer and efficiently encodes static images into powerful representations with temporal dynamics while minimizing redundancy. Our method has been extensively evaluated through experiments, demonstrating its effectiveness with state-of-the-art results: CIFAR10 (96.46\%), CIFAR100 (80.45\%), and ImageNet (70.42\%). We hope our investigations pave the way for more advanced coding schemes and inspire the design of high-performance and efficient spike-driven SNNs.

\section{Supplementary Material B}
In this supplementary material B, we first present the hyper-parameters settings and training details. Then we present the details of our architecture, training algorithm, and the analysis of the GAC’s firing rate, the period of direct coding \cite{wu2019direct}, and other coding schemes' dynamics. Moreover, we give additional results to depict the GAC's advantage.
\section{Hyper-Parameters Settings}
In this section, we give the specific hyperparameters of the LIF model and training settings in all experiments, as depicted in Tab. \ref{tab:config} and Tab.\ref{tab:hyperparameters}.
\begin{table}[htbp]
    \centering
    \caption{Hyper-parameter setting on LIF model.}
    \label{tab:config}
    \begin{tabular}{cc}
     \toprule 
    \multicolumn{1}{c}{Parameter} & \multicolumn{1}{c}{Value} \\ \midrule
    Threshold $\boldsymbol V_{th}$                           & 0.5                       \\
    Reset potential $\boldsymbol V_{reset}$                      & 0                         \\
    Decay factor $\tau$                           &0.25                    \\

    Surrogate function's window size $a$                             & 1                         \\ \bottomrule
    
    \end{tabular}
\end{table}

\begin{table}[htbp]
    \centering
    \caption{Hyper-parameter training settings of GAC-SNNs.}
    \begin{tabular}{lccc}
        \toprule 
         Parameter & CIFAR10 & CIFAR10 & ImageNet  \\
         \midrule
         Learning Rate & $5e-4$  & $5e-4$ & $1e-4$  \\
         Batch Size & 64 & 64 & 256 \\
         Time steps  & 6/4/2& 6/4/2 & 6/4 \\
         Training Epochs & 250 & 250 & 250 \\
         \bottomrule
    \end{tabular}
    \label{tab:hyperparameters}
\end{table}

\section{Training Details}
For all of our experiments, we use the stochastic gradient descent optimizer with 0.9 momentum and weight decay (e.g., CIFAR10 $5e-4$ and ImageNet $1e-4$ for 250 training epochs. The learning rate is set to 0.1 and cosine decay is to 0.  All the training and testing codes are implemented using PyTorch \cite{paszke2019pytorch}. The training is performed on eight NVIDIA V100 GPUs and 32GB memory for each task. The summaries of datasets and augmentation involved in the experiment are listed below.
\par
\noindent\textbf{CIFAR 10/100} consist of 50k training images and 10k testing images with the size of 32 × 32 \cite{krizhevsky2009learning}. We use ResNet-18 for both CIFAR10 and CIFAR100. Random horizontal flips, crops, and Cutmix are applied to the training images for augmentation. \par
\noindent\textbf{ImageNet}  contains around 1.3 million 1, 000-class images for training and 50, 000 images for validation \cite{krizhevsky2017imagenet}. The batch size and total epochs are set to 256. We crop the images to 224×224 and use the standard augmentation for the training data, which is similar to the Att-MS-ResNet\cite{yao2023attention}. Moreover, we use the label smoothing \cite{szegedy2016rethinking} to avoid gradient vanishing, which is similar to the Att-MS-ResNet \cite{yao2023attention}.
\section{Architecture Details}
The MS-ResNet-series architecture is the primary network architecture used in this work. There are numerous ResNet variations utilized in the field of SNNs, such as MS-ResNet-18, and MS-ResNet-34. They are also used interchangeably in existing literature \cite{hu2021advancing,yao2023attention}. To avoid confusion in various ResNet architectures, we will sort out the architecture details in this part.  Specifically, MS-ResNet-18 is originally designed for the ImageNet dataset in \cite{hu2021advancing}. To process CIFAR datase,
we remove the max-pooling layer, replace the first 7 $\times$ 7 convolution layer with a 3 $\times$ 3 convolution layer, and replace the first and the second 2-stride convolution operations as 1-stride, following the
modification philosophy from TET \cite{deng2022temporal}.

\begin{table}[htbp]
    \caption{MS-ResNet-series architecture for ImageNet.\label{tab:ImageNet_structure}}
    \centering
     \begin{adjustbox}{max width=\linewidth} 
    \begin{tabular}{c|c|c|c}
    \hline
    Stage & Output Size   & ResNet-18  & ResNet-34  \\ \hline 
    Conv1 & 112x112    & \multicolumn{2}{c}{7x7, 64, stride=2}  \\ \hline 
    Conv2 & 56x56                & $\left[\begin{array}{c}\text{3x3, 64}\\ \text{3x3, 64}\end{array}\right]*2$           & $\left[\begin{array}{c}\text{3x3, 64}\\ \text{3x3, 64}\end{array}\right]*2$            \\ \hline 
    Conv3 & 28x28              &$\left[\begin{array}{c}\text{3x3, 128}\\ \text{3x3, 128}\end{array}\right]*2$            & $\left[\begin{array}{c}\text{3x3, 128}\\ \text{3x3, 128}\end{array}\right]*4$            \\\hline
    Conv4 & 14x14             &$\left[\begin{array}{c}\text{3x3, 256}\\ \text{3x3, 256}\end{array}\right]*2$            & $\left[\begin{array}{c}\text{3x3, 256}\\ \text{3x3, 256}\end{array}\right]*6$            \\\hline
    Conv5 & 7x7              &$\left[\begin{array}{c}\text{3x3, 512}\\ \text{3x3, 512}\end{array}\right]*2$            & $\left[\begin{array}{c}\text{3x3, 512}\\ \text{3x3, 512}\end{array}\right]*3$            \\\hline
    FC    & 1x1         & \multicolumn{2}{c}{AveragePool, FC-1000} \\\hline
    \end{tabular}
    \end{adjustbox}
\end{table}
\vspace{-1mm}
\section{ Analysis of the Firing Rate}
GAC controls membrane potentials and reduces spike firing rates of the post-encoding layer in GAC-SNN. It achieves by transforming matrix multiplication into sparse addition, which can be implemented as addressable addition on neuromorphic chips. Moreover, high sparsity is observed in GAC-based MS-ResNet, as shown in Fig. \ref{energy}. The firing rate $r$ is defined as the firing probability of each neuron per timestep and can be estimated by:
\begin{equation}
    r = \frac{\# \text{Spike}}{\# \text{Neuron} \cdot T},
\end{equation}
Where \#Spike denotes the number of spikes during
$T$ timesteps and \#Neuron denotes the number of neurons
in the network. Specifically, after applying the GAC technique, the average spike firing rate of MS-ResNet-18 decreased from \textbf{0.215} to \textbf{0.154}. Similarly, the release rate of MS-ResNet-34 reduced from \textbf{0.225} to \textbf{0.1729}.  These reductions indicate that GAC has effectively contributed to reducing the spike firing rates in both MS-ResNet-34 and MS-ResNet-104 models. Lower spike firing rates can be beneficial in terms of reducing computational load, and energy consumption, and potentially improving the overall efficiency of the models.

\begin{figure}[htbp]
    \centering
    \includegraphics[width=0.50\textwidth]{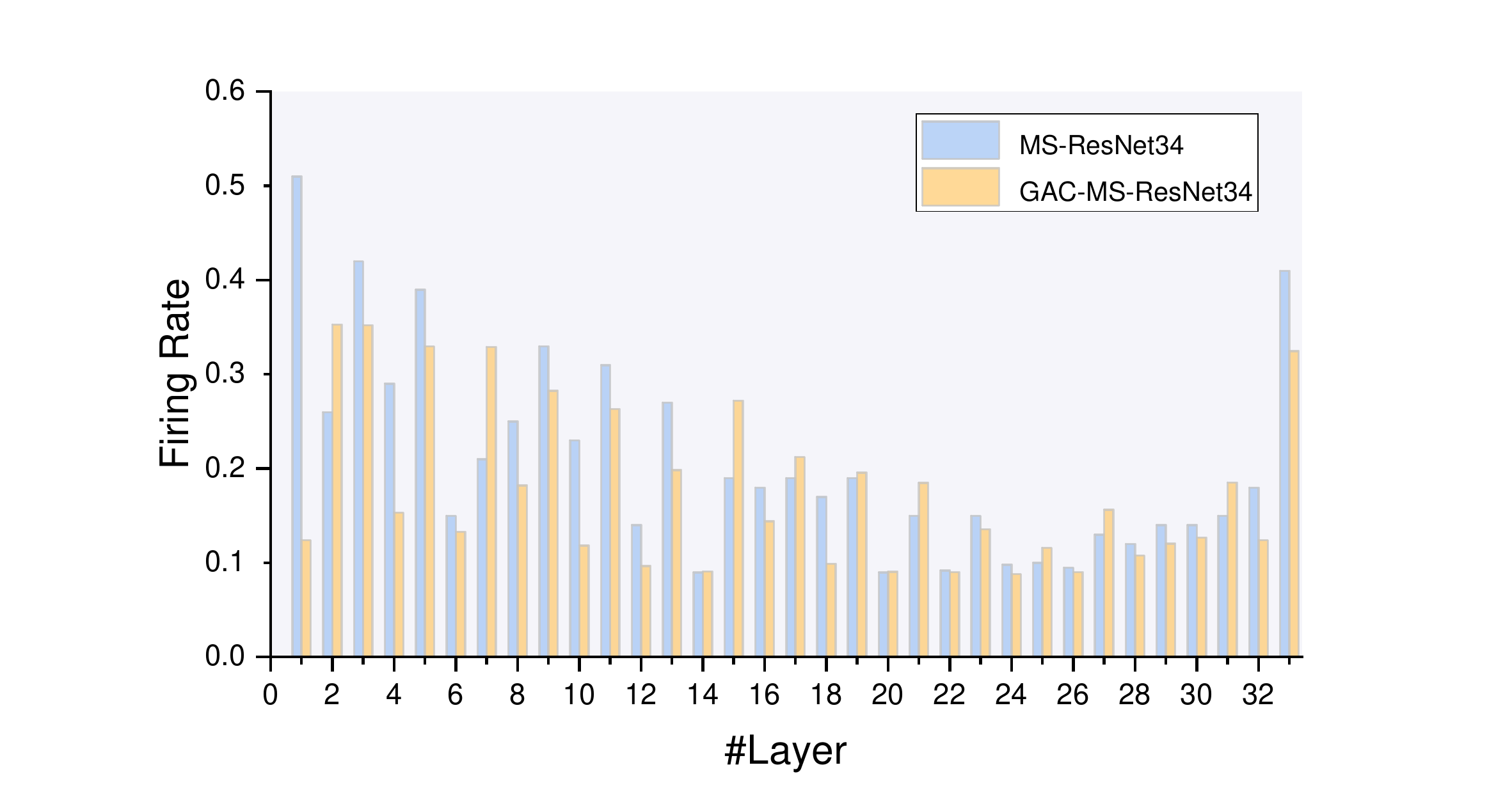}
    \caption{Firing rate advantage on the ImageNet dataset.}
    \label{energy} 
\end{figure}
\vspace{-1mm}
\vspace{-1mm}
\section{Derivation of the direct coding period $\boldsymbol T_p$}
For direct coding, the real value data repeatedly pass through the \{\textit{Conv-BN}\} layer for linear transformations and then the LIF model for $\{0, 1\}$ encoding results. Since the input of the LIF model is the same for every time step, it is trivial that the output of the LIF model in the encoding layer is periodical and the information of input (real value) is encoded into the period $\boldsymbol T_{p}$. Suppose the reset value is $0$ and the threshold is $\boldsymbol V_{th}$, the time $\boldsymbol T_{p}$ is expressed as follow:
\begin{equation}
\label{eq:period_encoding}
\boldsymbol T_{p} = \lceil\log_{\tau}(1 - \frac{\boldsymbol V_{th}(1- \tau )}{x_{i,j}})\rceil,
\end{equation}
Here $\tau$ is the attenuation factor and $x_{i,j}$ is the pixel of the input image of the LIF model after the \{\textit{Conv-BN}\} module. $\boldsymbol T_{p}$ is the first time step of firing and the period of firing for the corresponding LIF model. It is trivial that $T$ increases monotonically with $x$ and the resolution of the encoding depends on the value of $x$. The most important thing is that after the LIF model, the direct coding is period coding, making the $0, 1$ output periodical.
Moreover, the derivation of the direct coding's period is as follows:
Suppose when $t=0$, the membrane potential $\boldsymbol U^{t}=0$, and the neuron fires at time step $t=\boldsymbol T_{p}$. According to the direct coding, the input of any specific neuron is $x_{i,j}$ (a constant). Since the threshold and the attenuation factor are separately denoted as
$\boldsymbol V_{th}$ and $\tau$.According to the iterative formula of the LIF model before the spike fire time, we have:
\begin{equation}
\label{eq:serial_sum}
\boldsymbol U^{t} = \tau \boldsymbol U^{t-1} + x_{i,j} = \sum_{k=0}^{t-1}\tau^{k}x_{i,j} = x_{i,j}\sum_{k=0}^{t-1}\tau,
\end{equation}
Also, we have:
\begin{equation}
\label{eq:edge_membrane}
\boldsymbol U^{T_{p}-1}\leq \boldsymbol V_{th} \leq \boldsymbol U^{T_{p}},
\end{equation}
Bringing \ref{eq:serial_sum} to \ref{eq:edge_membrane}, we have:

\begin{align}
&x_{i,j}\sum_{k=0}^{T_{p}-2}\tau\leq \boldsymbol V_{th} \leq x_{i,j}\sum_{k=0}^{T_{p}-1},\\
&(1-\tau^{T_{p}-1})\frac{x}{1-\tau}\leq \boldsymbol V_{th} \leq (1-\tau^{T_{p}})\frac{x_{i,j}}{1-\tau},\\
&\tau^{T_{p}} \leq 1 - \frac{\boldsymbol V_{th}(1-\tau)}{x_{i,j}} \leq \tau^{T_{p}-1},\\
& T_{p}-1 \leq \log_{\tau}(1-\frac{\boldsymbol V_{th}(1-\tau)}{x_{i,j}}) \leq T_{p},\\
& T_{p} = \lceil\log_{\tau}(1 - \frac{\boldsymbol V_{th}(1- \tau )}{x_{i,j}})\rceil,
\end{align}

\section{Analysis of other Coding Schemes' Dynamics}
In this section, we use the observer model to analyze the dynamics of common coding schemes such as rate coding \cite{van2001rate} and coding scheme used in MS-ResNet  \cite{hu2021advancing}. Hence we give the following two propositions to describe, i.e.,
\begin{myproposition}
\label{appendix_prop1}
For rate coding with positive value $x\in(0, 1)$, giving the total timestep $\boldsymbol T$, denoting dynamics duration of rate coding as $\boldsymbol T_{r}$, then $\boldsymbol T_{r}=\boldsymbol T$.
\end{myproposition}
\begin{proof}
For rate coding, the output in each timestep follows the Bernoulli distribution of $n=1, p=x$, and is independent of each other at different time steps (that is, independent and identically distributed). Therefore, for any timestep $t$, information entropy $H(\boldsymbol{V}^{t})=x\log(x)+(1-x)\log(1-x) > 0$. This means that all-time step encoding information is dynamic.
\end{proof}

\begin{myproposition}
\label{appendix_prop2}
For MS coding, denoting dynamics duration of rate coding as $\boldsymbol T_{m}$, then $\boldsymbol T_{m}=\boldsymbol 1$.
\end{myproposition}
\begin{proof}
For MS coding, the information for each timestep is repeated. Thus, it is predictable after the first time of observation. Therefore, $T_{m}=1.$
\end{proof}
According to \textbf{Proposition} 1, it is known that rate coding has all-time step dynamics. However, it suffers from long time steps to remain high performance, while small time steps result in lower representation resolution. According to \textbf{Proposition} 2, MS coding results are the same at every time step. And our GAC-based MS-ResNet improves temporal dynamics and encoding efficiency through an attention mechanism.
\vspace{-1mm}
\section{Training Algorithm to Fit Target Output}
In this section, we introduce the training process of SNN gradient descent and the parameter update method of STBP \cite{wu2018spatio}. SNNs' parameters can be taught using gradient descent techniques, just like ANNs, after determining the derivative of the spike generation process. Classification, as well as other tasks for both ANNs and SNNs, can be thought of as optimizing network parameters to meet a goal output when given a certain input. Moreover, the accumulated gradients of loss $\mathcal{L}$ with respect to weights $\boldsymbol{W}^{j}_{n}$ at layer $n$ can be calculated as:

\begin{equation}
    \begin{cases}
     \frac{\partial \mathcal{L}}{\partial \boldsymbol{S}^{t, n}_{i}}=\sum_j \frac{\partial \mathcal{L}}{\partial \boldsymbol{U}_{j}^{t, n+1}} \frac{\partial \boldsymbol{U}_{j}^{t, n+1}}{\partial \boldsymbol{S}_{i}^{t, n}}+\frac{\partial \mathcal{L}}{\partial \boldsymbol{U}_{i}^{t+1, n}} \frac{\partial \boldsymbol{U}_{i}^{t+1, n}}{\partial \boldsymbol{S}_{i}^{t, n}}\\
\frac{\partial \mathcal{L}}{\partial \boldsymbol{U}_{i}^{t, n}}=\frac{\partial \mathcal{L}}{\partial \boldsymbol{S}^{t,n}_{j}} \frac{\partial \boldsymbol{S}^{t,n}_{j}}{\partial \boldsymbol{U}^{t,n}_{j}}+\frac{\partial \mathcal{L}}{\partial \boldsymbol{U}^{t+1, n}_{j}} \frac{\partial \boldsymbol{U}^{t+1, n}_{j}}{\partial \boldsymbol{U}^{t, n}_{j}}
\\
\frac{\partial \mathcal{L}}{\partial \boldsymbol{W}^{j}_{n}}=\sum^{T}_{t=1}\frac{\partial \mathcal{L}}{\partial \boldsymbol{U}_{i}^{t, n+1}}\boldsymbol{S}^{t,n}_{j},
    \end{cases} 
    \label{BPTT}
\end{equation}
where $\boldsymbol{S}^{t,n}$ and $\boldsymbol{U}^{t,n}_{j}$ represent the binary spike and membrane potential of the neuron in layer $n$, at time $t$. 
\\
Moreover, notice that $ \frac{\partial \boldsymbol{S}^{t, n}}{\partial \boldsymbol{U}^{t, n}}$ is non-differentiable. To overcome this problem,  \citeauthor{wu2018spatio} (\citeyear{wu2018spatio}) proposes the surrogate function to make only the neurons whose membrane potentials close to the firing threshold receive nonzero gradients during backpropagation.  In this paper, we use the rectangle function, which has been shown to be effective in gradient descent and may be calculated by:
\begin{equation}
\label{eq3}
    \frac{\partial \boldsymbol{S}^{t, n}}{\partial \boldsymbol{U}^{t, n}}=\frac{1}{a} \operatorname{sign}\left(\left|\boldsymbol{U}^{t, n}-\boldsymbol{V}_{\mathrm{th}}\right|<\frac{a}{2}\right),
\end{equation}
where $a$ is a defined coefficient for controlling the width of the gradient window.

% \begin{algorithm}[htbp]
% \caption{Overall Training Algorithm}
% \label{alg:algorithm}
% \textbf{Input}: Input image $X$; target output $Y$ = \{$Y^1 , Y^2 , \cdots, Y^t$ \}; learning rate $\eta$; network’s parameter $ W$; simulating time-steps $ T$;\\
% \textbf{Output}: network's average output $O$
% \begin{algorithmic}[1] %[1] enables line numbers
% \STATE \textbf{Function} \Call{Training} {$X$}
% \STATE Initialize network’s parameters $W$
% \STATE Using GAC to generate encoded feature sequences $X^{1:T}$ 
% \STATE Create an empty list $S$ = \{\}
% \For{$t=1$ to $t=T$}\\
% \STATE $~~~~~$ Input $X^t$ to network, get output spikes $S^t$
% and append $S^t$ to $S$ = \{$S^1 , S^2 , \cdots, S^t$ \}
% \STATE \textbf{endfor} 
% \STATE Calculate network's average output $O$ = $\frac{1}{T}$ $\sum_{t=1}^{T}S^{t}$

% \STATE Calculate loss $L$ = $\mathcal{L}(O, Y)$
% \STATE Update parameter $W$= $W$- $\eta$$\cdot\nabla_{\boldsymbol{\theta}}L$ by Eq. \ref{BPTT}
% % \STATE \textbf{return} $\hat{x}$
% \end{algorithmic}
% \end{algorithm}

\section{Additional Results}
\subsection{Convergence  Analysis}
We empirically demonstrate the convergence of our proposed method. As shown in Fig. \ref{fig:conver}, the performance of our GAC-based MS-ResNet stabilizes and converges to a higher level compared to MS-ResNet as training epochs increase. Moreover, the GAC-SNN achieves state-of-the-art performance after only \~20 epochs, demonstrating its efficacy.
\begin{figure}[htbp]
    \centering
    \includegraphics[width=0.95\columnwidth]{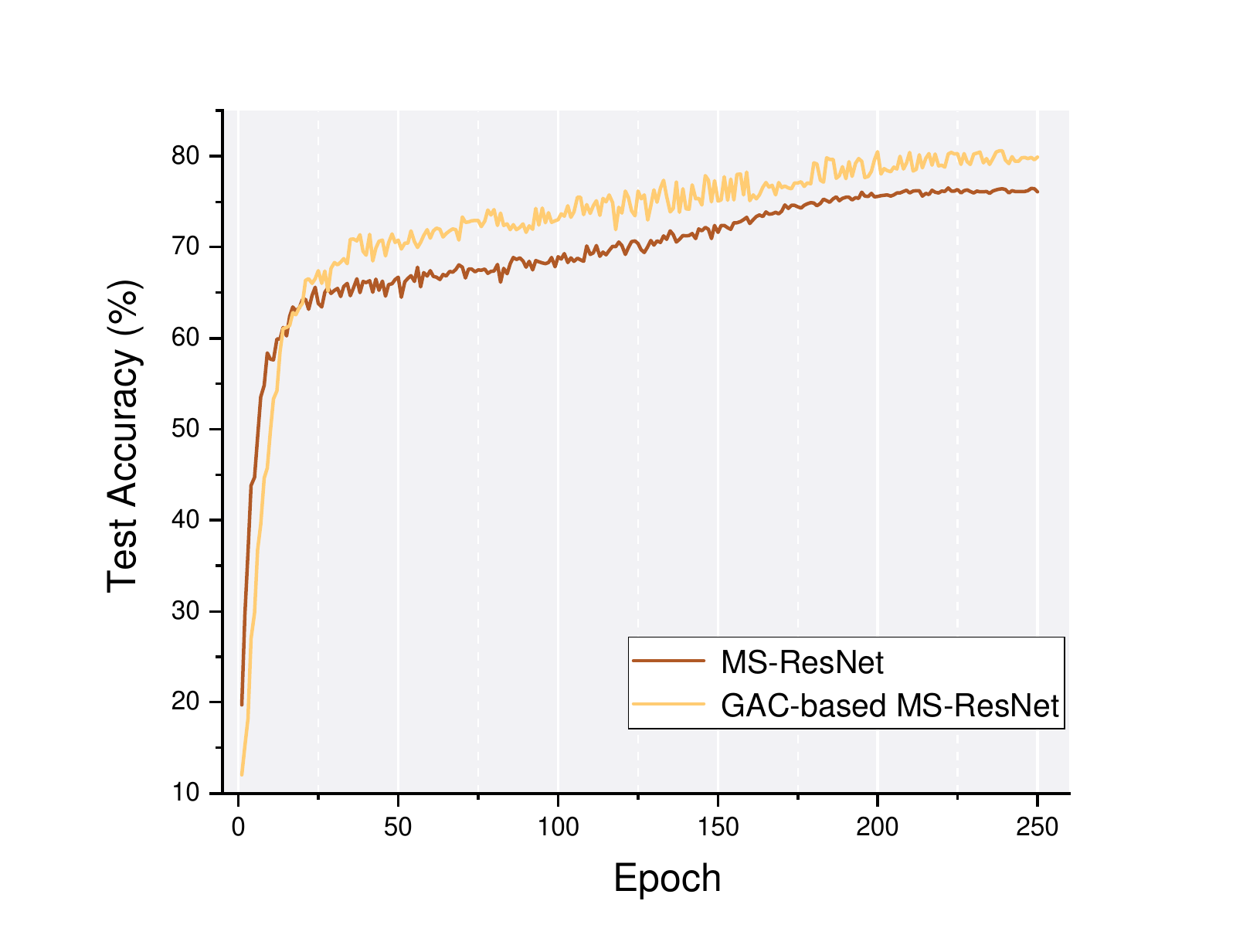}
    \caption{Convergence of compared SNN methods on CIFAR100 dataset}
    \label{fig:conver}
\end{figure}
\subsection{Trainable Parameter Analysis}
We also give the effects of the Temporal Attention (TA), Spatial Channel Attention (SCA), and Gated Attention Unit (GAU)  on the increase of model parameters, as shown in Tab. \ref{tab: parameter_incre}. First, with different datasets, the proportion of different modules to the number of model parameters varies. The increase of SCA and GAU to the number of parameters is much higher than that of TA. This phenomenon is consistent with our expectations that given that the channel size in the dataset is much larger than the simulation time step, indicating that SCA has a lot of parameters. 
\begin{table}[htbp]
\centering
\caption{The increase of TA, SCA, and GAU modules on the number of model parameters. An intuitive impression is that TA has little effect on the number of model parameters, while the insertion of SCA and GAU leads to a significant increase in the number of model parameters.}
{
\begin{tabular}{lccc}
\toprule

Datasets  & TA & SCA & GAU \\\midrule
CIFAR10                               & 0.0001\%    & 1.0287\%     & 1.0291\%             \\
CIFAR100                             & 0.0001\%       & 1.0259\%      & 1.0262\%        \\
ImageNet                            & 0.0001\%       & 0.5925\%      & 0.5927\%       \\
\bottomrule     
\end{tabular}}
\label{tab: parameter_incre}
\end{table}
% \subsection{Running time Comparison}
% In practical application, we can thus make a trade-off based on the actual situation: while considering energy consumption, we can only use the SCA mechanism in the model, but when an accuracy guarantee is required, the TCJA mechanism can be utilized. However, due to the efficient implementation of convolution, our approach does not need additional training time, which can be seen in Fig. 1.
\subsection{More Visualization Results on  GAC}
To further illustrate the advantages of our GAC, we provide the visualization results (i.e.,  model's GradCAM \cite{selvaraju2017grad}) on MS-ResNet-34 with GAC or direct coding, as shown in Fig. \ref{fig:gradcam} to help understand our design. And it can be seen in GradCAM heat map results that our GAC can help SNNs to capture more texture information.
\begin{figure}[htbp]
    \centering
    \includegraphics[width=0.95\columnwidth]{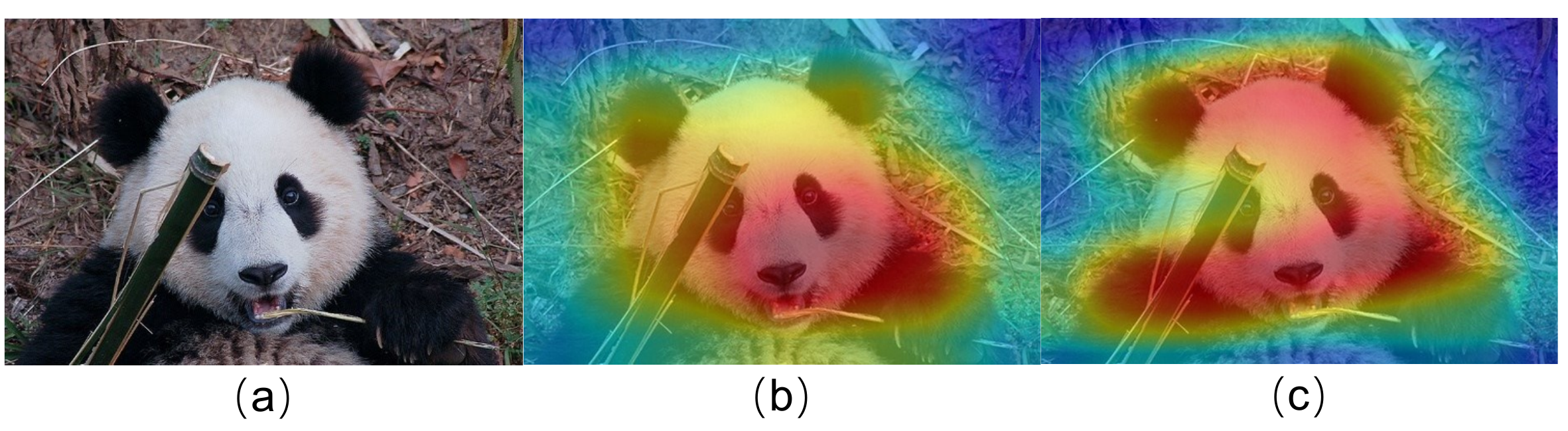}
    \caption{(a): original image. (b)(c) GradCAM results of without/with GAC on MS-ResNet-34.}
    \label{fig:gradcam}
\end{figure}
\bibliography{aaai24}

% \begin{figure}[htbp]
%     \centering
%     \includegraphics[width=0.97\columnwidth]{AnonymousSubmission/LaTeX/comparison.png}
%     \caption{Performance advantage on the CIFAR100 dataset.}
%     \label{attention}
% \end{figure}
\end{document}